\documentclass[runningheads]{llncs}

\usepackage{booktabs} 
\usepackage{graphicx, subfigure}

\usepackage{array}
\usepackage{mathtools}
\usepackage{multirow}
\usepackage{enumitem}
\usepackage{dsfont}
\newcommand{\nosection}[1]{\vspace{2pt}\noindent\textbf{#1.}}
\usepackage[mathscr]{euscript}
\usepackage{amssymb}  
\usepackage[linesnumbered,boxed,ruled]{algorithm2e}

\begin{document}
\title{Secret Sharing based Secure Regressions with Applications}
%

\author{Chaochao Chen$^{1}$ \and Liang Li$^{1}$ \and Wenjing Fang$^{1}$ \and Jun Zhou$^{1}$ \and Li Wang$^{1}$ \and \\Lei Wang$^{1}$ \and Shuang Yang$^{1}$ \and Alex Liu$^{1}$ \and Hao Wang$^{2}$ \\
\{chaochao.ccc,liangli.ll,bean.fwj,jun.zhoujun,raymond.wangl\}@antfin.com\\
\{shensi.wl,shuang.yang,alexliu\}@antfin.com, wanghao@sdnu.edu.cn
}
\authorrunning{Chaochao Chen et al.}
\institute{$^{1}$Ant Financial Services Group, $^{2}$Shandong Normal University \\}

\maketitle              
\begin{abstract}
Nowadays, the utilization of the ever expanding amount of data has made a huge impact on web technologies while also causing various types of security concerns. On one hand, potential gains are highly anticipated if different organizations could somehow collaboratively share their data for technological improvements. On the other hand, data security concerns may arise for both data holders and data providers due to commercial or sociological concerns. 
To make a balance between technical improvements and security limitations, we implement secure and scalable protocols for multiple data holders to  train linear regression and logistic regression models. We build our protocols based on the secret sharing scheme, which is scalable and efficient in applications. Moreover, our proposed paradigm can be generalized to any secure multiparty training scenarios where only matrix summation and matrix multiplications are used. 
We demonstrate our approach by experiments which shows the scalability and efficiency of our proposed protocols, and finally present its real-world applications.

\end{abstract}


\section{Introduction}
With the ever expanding collection of data, there have been increasing concerns on the security and privacy issues when utilizing big data to facilitate technological improvement. On the customer side, privacy concerns arise when individual information is collected with potential risk of leakage \cite{chen2018privacy}, while on the collector side, there are security concerns since the collectors are always willing to protect their own resources including data.

Meanwhile, there is increasingly potential gains in terms of analytical power if different organizations could collaboratively combine their data assets for data mining or information retrieval. 
For example, health data from different hospitals can be used together to facilitate more accurate diagnosis \cite{li2019federated}, while financial companies can collaborate to train more effective fraud-detection engines \cite{zheng2020industrial}.

The major task is to combine data from multiple entities so as to improve model training performance while still protecting data securities for individual holders from any possibility of information disclosure, i.e., \textit{shared machine learning} in literature \cite{wu2019generalization,liu2020privacy,zheng2020industrial,chen2020secure,chen2020practical}. 
The security and privacy concern, together with the desire of data combination poses an important challenge for both academia and industry. 

To derive an application-oriented solution to the above challenges, we design secure multiparty protocols based on the classical secret sharing scheme, which is efficient in both computational and communicational costs. We restrict our attention to the simple yet widely used regression models in industry, i.e., linear regression and logistic regression, and consider two scenarios (vertical and horizontal) of data partitioning among data holders. 




\subsection{Our Contribution}
We summarize the main contributions as follows. 
\begin{itemize} [leftmargin=*] \setlength{\itemsep}{-\itemsep}
\item We make a computational reduction from the multiparty training of regression models to secure multiparty matrix summation and secure two party matrix multiplications. 
This paradigm is quite different from previous work in this area \cite{mohassel2017secureml} since for the gradient descent step in Equation \eqref{batch-update}, although the calculation of subtraction can be implemented easily via secret sharing, the matrix multiplications can not be easily generalized to arbitrary number of parties and so their training protocols are not applicable to multiparty cases. 
Our paradigm takes the secure two-party matrix multiplication protocol as a complete `black-box', which means any secure matrix product protocols can be ``embedded'' into our paradigm to train regression models. 
By reducing the matrix multiplications for multiple data holders into matrix multiplications for two parties, we can even generalize our paradigm to any training process other than linear regression and logistic regression. As long as the training process involves only summation (subtraction) and multiplication, our method can be applied. 
We also remark that for some special functions like sigmoid functions, we can apply polynomial approximations to transform the special functions into summation and multiplication.
\item We conduct empirical studies on our proposed method and experimentally verify that our protocols achieve exactly the same performance while the additional overhead for the computation and communication costs scale linearly with data size. This demonstrates that our proposed method can be applied into large scale datasets.
\end{itemize}

\subsection{Related Work}
Secure Multi-Party Computation (MPC) was initiated in \cite{yao1982protocols}, which aims to generate methods (or protocals) for multi-parties to jointly compute a function (e.g., vector multiplication) over their inputs (e.g., vector for each party) while keeping those inputs private. 
Different schemes can be used to design MPC protocols, such as garbled circuits \cite{yao1986generate} and secret sharing \cite{shamir1979share}. 
MPC makes the secure collaboration between different data holders possible, and has been applied to implement many secure machine learning algorithms, such as decision tree \cite{lindell2005secure}, linear regression \cite{nikolaenko2013privacy,GasconSB0DZE17,mohassel2017secureml,Karr2010a,hall2011secure}, logistic regression \cite{demmler2015aby,mohassel2017secureml,kim2018secure}, neural networks \cite{mohassel2017secureml,zheng2020industrial}, and recommender system \cite{chen2020secure,chen2020practical}. 

The most similar work to ours is SecureML \cite{mohassel2017secureml}, however, we are different from them in mainly three aspects. 
First, since the SGD update step in \cite{mohassel2017secureml} is fully secret-shared for each matrix operation in Equation \eqref{batch-update}, secure matrix product cannot be easily generalized to multiple parties. 
Second, SecureML adopts the techniques of arithmatic sharing, boolean sharing, and Yao sharing for secure machine learning. Although it provides fruitful protocols for academic researchers to build any machine learning models, it also increases the difficulty for industrial practitioners deploy customized models. 
Thirdly, our protocol under vertically data partition setting has less communication complexity, since we only need to secretly share models and labels instead of features. 
These are the main differences between the training protocol in \cite{mohassel2017secureml} and ours. 

\section{Preliminaries}

\subsection{Linear Regression and Logistic Regression}

\nosection{Linear Regression}
Linear gression is popularly used in industry due to its simply yet robust ability.  
Given $m$ data samples $\textbf{x}_i \in \mathbb{R}^d$ each with $d$ dimensional input features and output labels $y_i\in \mathbb{R}$, \textit{regression} is a statistical method to learn a fitting function $f: \mathbb{R}^d \rightarrow \mathbb{R}$ such that $f(\textbf{x}_i) \approx y_i$ for all $i \in [m]$. 
Note that, without loss of generality, we use italic text (e.g., $y_i$) for variables, lower-case bold letters (e.g., $\textbf{w}$) for vectors, and upper-case bold letters (e.g., \textbf{X}) for matrices. 
In linear regression, we assume that $f$ can be represented as a linear combination of the input features, i.e., $f(\textbf{x}_i) = \sum_{j = 1}^{d} x_{ij}w_j = \textbf{x}_i\cdot\textbf{w}$, where $x_{ij}$ is the $j$th feature of sample $\textbf{x}_i$ and $w_j$ is the $j$th coefficient of $\textbf{w}$. The value $\hat{y}_i = f(\textbf{x}_i)$ is called the \textit{predictive value} on $\textbf{x}_i$.
To learn the coefficient vector $\textbf{w}$, square loss is commonly used:
\begin{equation}
L_{li}(\textbf{w}) = \frac{1}{2}\sum_{i = 1}^m(\hat{y}_i - y_i)^2,
\end{equation}
which measures the total difference between the predictive values and the true labels. The objective is to minimize this loss function, which is also known as the \textit{least squares estimation}.

\nosection{Logistic Regression}
Logistic regression is also widely used for binary classification tasks in practice, in which each data sample $\textbf{x}_i \in \mathbb{R}^d$ has an output label $y\in \{0, 1\}$. The fitting function $f$, also called the \textit{activation function} in this setting, can be represented as the \textit{logistic function} of the weighted combination of the input features: $f(\textbf{x}_i) = \frac{1}{1 + e^{-\textbf{x}_i\cdot \textbf{w}}}$.

To train the coefficients $\textbf{w}$ of the model, an information-theoretic loss function called \textit{cross entropy} is used: 
\begin{equation}
L_{lo}(\textbf{w}) = -\sum_{i = 1}^m\left(y_i\log \hat{y}_i + (1-y_i)\log (1-\hat{y}_i)\right),
\end{equation}
where $\hat{y}_i = f(\textbf{x}_i)$ is the predictive value of data sample $\textbf{x}_i$. The cross entropy measures the total distributional difference between the predictive values and the true labels over all data samples and the objective of logistic regression is to minimize it. Note that we omit the regularization terms for conciseness. 

\subsection{Mini-Batch Stochastic Gradient Descent}
Mini-batch stochastic gradient descent (SGD) is an efficient iterative algorithm for optimizing a global function. In linear regression and logistic regression, the goal is to find a best-fit coefficient vector $\textbf{w}$ which minimizes of the loss function. 
Let $\textbf{X}_B$ (resp. $\textbf{y}_B$) be the $|B|\times d$ (resp. $|B|\times 1$) submatrices of $\textbf{X}$ (resp. $\textbf{y}$), where the row indices are all from batch $B$. For linear regression and logistic regression, the update equation for batch SGD is as follows:
\begin{equation}\label{batch-update}
\textbf{w} := \textbf{w} - \frac{\alpha}{|B|}\textbf{X}_B^T\times(\hat{\textbf{y}_B} - \textbf{y}_B),
\end{equation}
where $\alpha$ is the learning rate which controls the moving magnitude. 

\subsection{Secure Multiparty Computation}
Suppose $n$ parties individually hold private data $x_1, x_2, \cdots, x_n$ and $g$ is a multivariate function with $n$ variables over arbitrary domain (e.g. boolean, integer, or real). Secure multiparty computation aims at designing protocols for the $n$ parties to collaboratively compute $g(x_1, x_2, \cdots, x_n)$ such that at the end of the protocol, each party gets the value $g(x_1, x_2, \cdots, x_n)$ without being able to explorer any information of other party's private data. 
Specifically, secure multi-party computation is a multi-party random process. 
Take two parties for example, it maps pairs of inputs (one from each party) to pairs of outputs (one for each party), while preserving several security properties, such as correctness, privacy, and independence of inputs \cite{hazay2010efficient}. This random process is called \emph{functionality}. Formally, denote a two-output functionality $f=(f_1,f_2)$ as $f:\{0,1\}^* \times \{0,1\}^* \rightarrow \{0,1\}^* \times \{0,1\}^*$. For a pair of inputs $(x,y)$, where $x$ is from party $P_1$ and $y$ is from party $P_2$, the output pair $(f_1(x,y),f_2(x,y))$ is a random variable. $f_1(x,y)$ is the output for $P_1$, and $f_2(x,y)$ is for $P_2$. 
During this process, neither party should learn anything more than its prescribed output.

There are fruitful methods such as garbled circuit and secret sharing for designing secure protocols, among which the secret sharing scheme is the most efficient one in practice. The very basic idea of secret sharing is that each data holder divide its own data into completely random shares and distribute the shares among different parities. In running the protocol, each party do local computations based on its shares coming from other parties and the final result can be recovered via global communication. 
In this paper, we use $\left\langle x \right\rangle_i$ to denote the $i$-th share of a secret $x$, which can be generalized to vectors and matrices. 
One should note that secret sharing only works in finite field, since random shares should be generated in uniform distribution to guarantee security. We use fix-point approximation following the existing research \cite{mohassel2017secureml}. 

\nosection{Security Model} 
In this paper, we consider the \emph{standard semi-honest model}, where a probabilistic polynomial-time adversary with semi-honest behaviors is considered. 
Take two parties for example, in this security model, the adversary may corrupt and control one party (referred as to {\em the corrupted party}), and try to obtain information about the input of the other party (referred as to {\em the honest party}). During the protocol execution with the honest party, the adversary will follow the protocol specification, but may attempt to obtain additional information about the honest party's input by analyzing the corrupted party's \emph{view}, i.e., the transcripts it receives during the protocol execution. In order to ensure correctness and privacy, the following formal security definition is proposed \cite{goldreich2009foundations}. 

\begin{definition}[Security in semi-honest model]
	
Let $f = (f_1, f_2)$ be a \emph{deterministic} functionality and $\pi$ be a two-party protocol for computing $f$. Given the security parameter $\kappa$, and a pair of inputs $(x,y)$ (where $x$ is from $P_1$ and $y$ is from $P_2$), the view of $P_i$ ($i=1,2$) in the protocol $\pi$ is denoted as $\mathsf{view}^\pi_i(x, y, \kappa) = (w, r_i, m_i^1,\cdots, m_i^t)$, where $w\in\{x, y\}$, $r_i$ is the randomness used by $P_i$, and $m_i^j$ is the $j$-th message received by $P_i$; the output of $P_i$ is denoted as $\mathsf{output}_i^\pi(x, y, \kappa)$, and the joint output of the two parties is $\mathsf{output}^\pi(x, y, \kappa) = (\mathsf{output}_1^\pi(x, y, \kappa) ), \mathsf{output}_2^\pi(x, y, \kappa))$.
We say that $\pi$ securely computes $f$ in semi-honest model if 
\begin{itemize}
	\item There exist probabilistic polynomial-time simulators $\mathcal{S}_1$ and $\mathcal{S}_2$, such that
$$\{\mathcal{S}_1(1^\kappa,x,f_1(x,y))\}_{x,y,\kappa}\cong\{\mathsf{view}_1^\pi(x,y,\kappa)\}_{x,y,\kappa},$$
$$\{\mathcal{S}_2(1^\kappa,y,f_2(x,y))\}_{x,y,\kappa}\cong\{\mathsf{view}_2^\pi(x,y,\kappa)\}_{x,y,\kappa}.$$
	\item The joint output and the functionality output satisfy
	$$\{\mathsf{output}^\pi(x, y, \kappa)\}_{x,y,\kappa}\cong\{f(x,y)\}_{x,y,\kappa},$$
\end{itemize}
where $x, y\in \{0, 1\}^*$, and $\cong$ denotes computationally indistinguishablity.
\end{definition}

\section{Secret Sharing based Secure Multiparty Regressions}

\subsection{Scenarios}\label{sub-sec-sce}

\begin{figure}
\centering
\subfigure [\emph{Data partitioned horizontally}]{ \includegraphics[width=5.5cm]{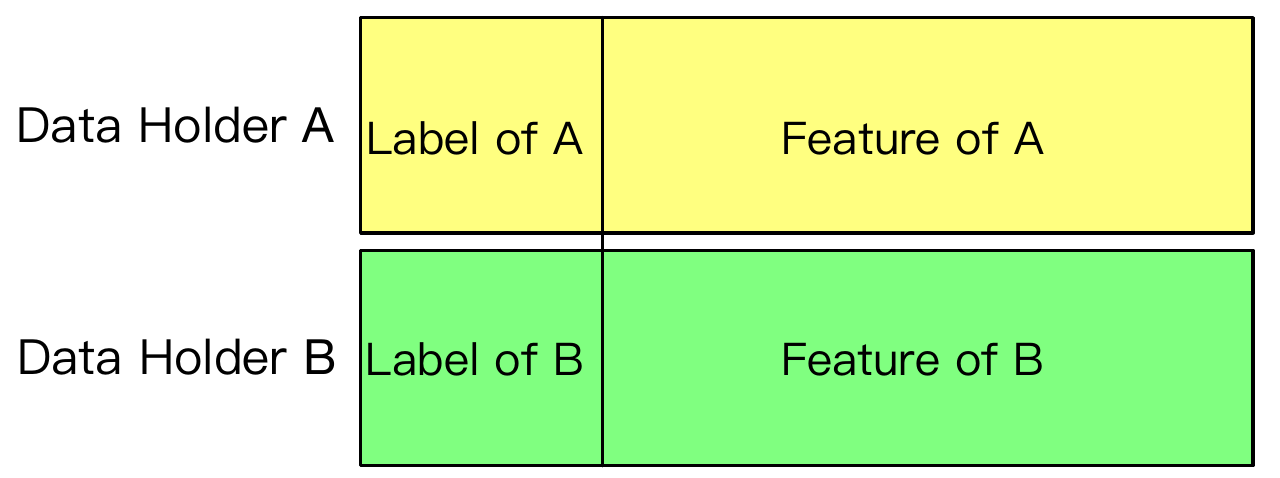}}~~~
\subfigure[\emph{Data partitioned vertically}] { \includegraphics[width=5.5cm]{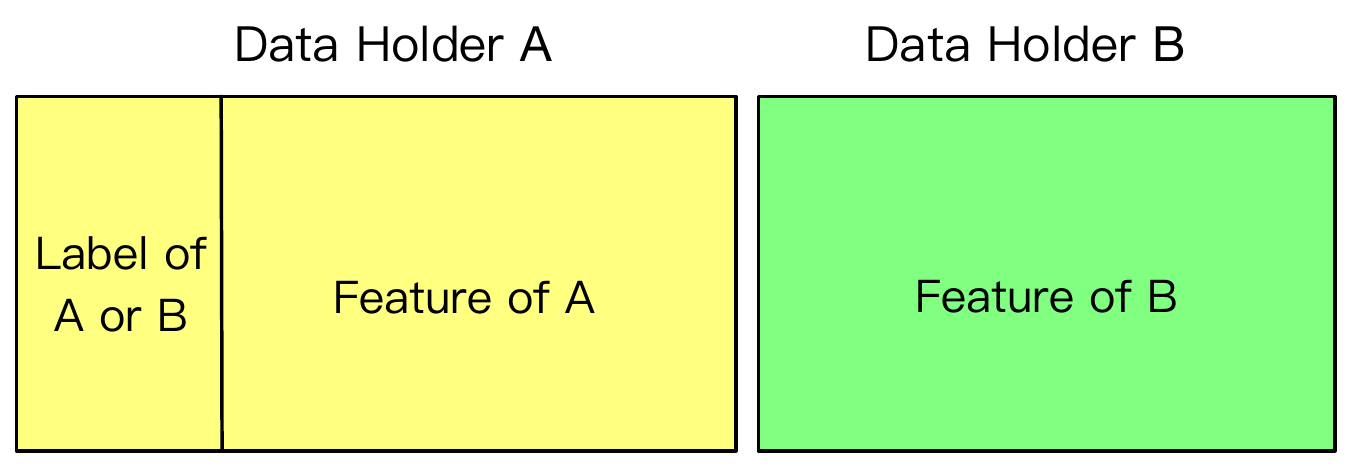}}
\caption{Data partition scenarios.}
\label{data-par}
\end{figure}

We mainly consider two scenarios of data partitioning, as is shown in Figure \label{data-par}. In the first scenario, data are \textit{partitioned horizontally} among multiple parties, which means each party holds a subset of data samples and the data in different parities have the same feature dimensions. In the second scenario, data are \textit{partitioned vertically}, where each party holds a subset of features over all samples and the sample indices of different parties have already been aligned. 

For notational convenience in describing the protocol, we use $\mathcal{A}_i$ to denote the $i$-th party involved in the protocol, where $i\in [n]$ and $n$ is the total number of all parties. Using secret sharing, we use $\left\{\left\langle \textbf{X}\right\rangle_i\right\}_{i\in[n]}$ to denote the set of shares of $\textbf{X}$ such that $\textbf{X} = \sum_{i=1}^n\left\langle \textbf{X}\right\rangle_i$. Here $\textbf{X}$ could be the original data, or secret-shares, or intermediate data during the algorithm procedure. $\left\langle \textbf{X}\right\rangle_i$ is called the $i$-share of data $\textbf{X}$, which is the fraction of data $\textbf{X}$ distributed to $\mathcal{A}_i$.

To securely train linear regression and logistic regression models, the crucial point is to design a protocol for multiple parties such that the SGD step in Equation \eqref{batch-update} can be conducted securely among different parties. We will describe our secure protocols for both scenarios in the following sections.

\begin{algorithm}[p]
\small
\caption{Linear regression protocol for horizontally partitioned data}\label{algo-horLR}
\KwIn {Feature matrices $\textbf{X}^i$ and label vectors $\textbf{y}^i$ for all party $\mathcal{A}_i(i\in[n])$, learning rate ($\alpha$), and maximum iterations ($T$)}
\KwOut{Coefficient vectors $\textbf{w}_i$ such that $\textbf{w} = \sum_{i=1}^n\textbf{w}_i$ where $\textbf{w}$ is the coefficient vector for the regression model}
\For{$i = 1$ to $n$ \texttt{in parallel}}
{
	$\mathcal{A}_i$ initializes $\textbf{w}_i$\\
}
\For{$t=1$ to $T$}
{   
		$n$ parties randomly select $i\in [n]$ via protocol\label{algostep-rand-batch}
	
		$\mathcal{A}_i$ locally samples batches $\textbf{X}^{i}_B$ and $\textbf{y}^i_B$ \\
		$\mathcal{A}_i$ locally generates $\left\{\left\langle\textbf{X}^{i}_B\right\rangle_j\right\}_{j\in [n]}$ and $\left\{\left\langle\textbf{y}^{i}_B\right\rangle_j\right\}_{j\in [n]}$\\
		$\mathcal{A}_i$ distributes $\left\{\left\langle\textbf{X}^{i}_B\right\rangle_j\right\}_{j\neq i}$ and $\left\{\left\langle\textbf{y}^{i}_B\right\rangle_j\right\}_{j\neq i}$ to other parties\\
		$\mathcal{A}_i$ locally calculates $\left\langle\textbf{X}^{i}_B\right\rangle_i \times \textbf{w}_i$ as $i$-share\\
		\For{$j \neq i$}
		{
			$\mathcal{A}_j$ locally calculates $\left\langle\textbf{X}^{i}_B\right\rangle_j \times \textbf{w}_j$ as $j$-share\\
			$\mathcal{A}_i$ and $\mathcal{A}_j$ calculate $i$-share $\left\langle\left\langle\textbf{X}^i_B\right\rangle_i\times \textbf{w}_j\right\rangle_i$ and $j$-share $\left\langle\left\langle\textbf{X}^i_B\right\rangle_i\times \textbf{w}_j\right\rangle_j$ via SMM protocol\\ \label{algostep-smm1}
			$\mathcal{A}_i$ and $\mathcal{A}_j$ calculate $i$-share $\left\langle\left\langle\textbf{X}^i_B\right\rangle_j\times \textbf{w}_i\right\rangle_i$ and $j$-share $\left\langle\left\langle\textbf{X}^i_B\right\rangle_j\times \textbf{w}_i\right\rangle_j$ via SMM protocol\\ \label{algostep-smm2}
			
		}
		\For{$j = 1$ to $n$ \texttt{in parallel}}
		{
			$\mathcal{A}_j$ locally calculates the summation of all $j$-shares, denoted as $\left\langle\textbf{X}^i_B\times \textbf{w}\right\rangle_j$\\
			$\mathcal{A}_j$ locally calculates $\textbf{err}_j = \left\langle\textbf{X}^i_B\times \textbf{w}\right\rangle_j - \left\langle\textbf{y}^i_B\right\rangle_j$\\
			$\mathcal{A}_j$ clears its $j$-shares to zero value\\
		}
		\For{$j = 1$ to $n$ \texttt{in paralle}}
		{
			$\mathcal{A}_j$ locally calculates $\left\langle\textbf{X}^i_B\right\rangle_j^T\times \textbf{err}_j$ as $j$-share\\
			\For{$k = 1$ to $n$ and $k\neq j$}
			{
				$\mathcal{A}_j$ and $\mathcal{A}_k$ calculates $j$-share $\left\langle\left\langle\textbf{X}^i_B\right\rangle_k^T\times\textbf{err}_j\right\rangle_j$ and $k$-share $\left\langle\left\langle\textbf{X}^i_B\right\rangle_k^T\times\textbf{err}_j\right\rangle_k$ via SMM protocol\\ \label{algostep-smm3}
			} 
		}
		
		\For{$j = 1$ to $n$ \texttt{in parallel}}
		{
			$\mathcal{A}_j$ locally calculates the summation of all $j$-shares, denoted as $\textbf{grad}_j$\\
			$\mathcal{A}_j$ locally updates $\textbf{w}_j$ by $\textbf{w}_j\leftarrow \textbf{w}_j - \frac{\alpha}{|B|}\cdot \textbf{grad}_j$\\
		} 
	
}
\end{algorithm}

\nosection{Threat Model}
Simialr as the existing researches \cite{mohassel2017secureml}, we use the \textit{semi-honest} (passive) adversary model, where the participants stricly follow the protocol, but may try to infer additional information from the middle messages during the protocol execution. Comparing with \textit{malicious} (active) adversary model, semi-honest adversary model enables the development of highly efficient secure computation protocols and has been widely used to develop secure machine learning applications \cite{demmler2015aby}.

\subsection{Protocols For Horizontally Partitioned Data}\label{sec-hor}

For secure linear regression and logistic regression models, the general idea is to locally sample a batch from one party in each iteration and do the computation in Equation \eqref{batch-update} via secret sharing. The crucial point for calculations in Equation \eqref{batch-update} is that every arithmetical operation (including matrix subtraction and multiplication) has to be secret-shared so that none of the other parties can get any information from the current party who generates the batch. 

\nosection{Secure Linear Regression Protocol} 
We first present the secure protocol for linear regresion under horizontally partitioned data in Algorithm \ref{algo-horLR}. 
As the data are horizontally partitioned among parties, we have 
\begin{eqnarray}\small
\textbf{X}^T &= \left(\left(\textbf{X}^1\right)^T, \left(\textbf{X}^2\right)^T, ..., \left(\textbf{X}^n\right)^T\right), \nonumber\\
\textbf{y}^T &= \left(\left(\textbf{y}^1\right)^T, \left(\textbf{y}^2\right)^T, ..., \left(\textbf{y}^n\right)^T\right), \nonumber
\end{eqnarray}
where $\textbf{X}^T$ denotes the transpose of matrix $\textbf{X}$, each sample in the feature matrix $\textbf{X}$ is a row vector and the label vector $\textbf{y}$ is a column vector.

The \textbf{for} loops with `\texttt{in parallel}' means that the $n$ parties should do the steps within the loop in parallel. SMM is short for \textit{Secure Matrix Multiplication.} 
Notice that when we call a SMM protocol, we assume that the parties each has one matrix whose dimensions are aligned to calculate the matrix product. 
Existing SMM protocols can be used as a black-box procedure in Line \ref{algostep-smm1}, \ref{algostep-smm2} and \ref{algostep-smm3}. Line \ref{algostep-rand-batch} is a protocol that determines whose data are selected to update model in the current batch, e.g., this can be simply sequential protocol that indicates all the participants' data are used to train the regression model sequentially. 


\noindent
\textbf{SMM protocols. }
SMM protocol makes sure that each of the parties (1) holds a secret share matrix such that the summation of the shares equals the matrix product, and (2) only gets the output, i.e., the secret share matrix, but not other middle information. 
To date, different SMM protocols have been proposed, they can be divided into two types, i.e., SMM with trusted initializer \cite{de2017efficient} and SMM without trusted initializer \cite{zhu2015ESP}. 
The main difference is that the former needs a trusted third-party to generate Beaver triples for the participates before the protocols starts, while the later one bypass the trusted initializer by sacrificing some security guarantee. 

\nosection{Secure Logistic Regression Protocol} 
One can slightly modify Algorithm \ref{algo-horLR} to change it from linear regression to logistic regression. 
The only difference is the calculation of predictive values, i.e., from $\textbf{x}_i\cdot\textbf{w}$ to $\frac{1}{1 + e^{-\textbf{x}_i\cdot \textbf{w}}}$. 
For logistic regression, to compute the logistic function using secret sharing, we approximate it by a $k$-order polynomial
\begin{equation}
\frac{1}{1 + e^{-z}} \approx \sum\limits_{j=0}^k q_j z^j.
\end{equation}

Existing researches proposed different coefficients ($q_j$) for logistic function \cite{aono2016scalable,chen2018logistic}. Here, we follow the 3-order polynomial in \cite{aono2016scalable}, i.e., $q_0 = 0.5$, $q_1 = 0.197$, $q_2 = 0$, and $q_3 = 0.004$. 
With this in mind, the logistic function can be easily approximated by using SMM protocol. 
We omit the details for conciseness. 

\subsection{Protocol For Vertically Partitioned Data}

\begin{algorithm}[p]
\small
\caption{Linear regression protocol for vertically partitioned data}\label{algo-verLR}
\KwIn {Feature matrices $\textbf{X}^i$ for party $\mathcal{A}_i(i\in[n])$, label vector $\textbf{y}$ located in party $\mathcal{A}_k$, learning rate ($\alpha$), and maximum iterations ($T$)}
\KwOut{Coefficient vectors $\textbf{w}_i$ such that $\textbf{w}^T = \left(\left(\textbf{w}_1\right)^T, \left(\textbf{w}_2\right)^T, \cdots, \left(\textbf{w}_n\right)^T\right)$ where $\textbf{w}$ is the coefficient vector for the regression model}

All $n$ parties agreed on a sequence of batches $B_1, B_2, \cdots, B_T$\label{algostep-batch-align}\\
\For{$i = 1$ to $n$}
{
	$\mathcal{A}_i$ initializes $\textbf{w}_i$ \\
	$\mathcal{A}_i$ generates shares $\{ \left\langle\textbf{w}_i\right\rangle_j \}_{j\in [n]}$ and distributes $\left\{\left\langle \textbf{w}_i \right\rangle_j\right\}_{j\neq i}$ to others\\
}
\For{$t=1$ to $T$}
{
	$\mathcal{A}_k$ generates shares $\left\{\left\langle\textbf{y}_{B_t}\right\rangle_j\right\}_{j\in [n]}$ and distributes $\left\{\left\langle\textbf{y}_{B_t}\right\rangle_j\right\}_{j\neq k}$ to others\\
	\For{$i = 1$ to $n$ }
	{
		$\mathcal{A}_i$ locally calculates $\textbf{X}^i_{B_t} \times \left\langle\textbf{w}_i\right\rangle_i$ as $i$-share \\
		\For{$j = 1$ to $n$ and $j \neq i$}
		{
			$\mathcal{A}_i$ and $\mathcal{A}_j$ calculate $i$-share $\left\langle \textbf{X}^i_{B_t} \times \left\langle\textbf{w}_i\right\rangle_j \right\rangle_i$ and $j$-share $\left\langle \textbf{X}^i_{B_t} \times \left\langle\textbf{w}_i\right\rangle_j \right\rangle_j$ via SMM protocol\\ 
		}
	}
	\For{$j = 1$ to $n$  \texttt{in parallel}}
	{
		$\mathcal{A}_j$ locally calculates the sum of all $j$-shares, denoted as $\left\langle\textbf{X}_{B_t}\times \textbf{w}\right\rangle_j$\\
		$\mathcal{A}_j$ locally calculates $ \left\langle \textbf{err} \right\rangle_j = \left\langle\textbf{X}_{B_t}\times \textbf{w}\right\rangle_j - \left\langle\textbf{y}_{B_t}\right\rangle_j$\\
	}

	\For{$i = 1$ to $n$ }
	{
		$\mathcal{A}_i$ locally calculates $\left\langle \textbf{err} \right\rangle_i ^T \times \textbf{X}^i_{B_t}$ as $i$-share of $\textbf{grad}_i$ \\
		\For{$j = 1$ to $n$ and $j \neq i$}
		{
			$\mathcal{A}_i$ and $\mathcal{A}_j$ calculate $i$-share $\left\langle \left\langle \textbf{err} \right\rangle_i ^t \times \textbf{X}^j_{B_t} \right\rangle_i$ and $j$-share $\left\langle \left\langle \textbf{err} \right\rangle_i ^t \times \textbf{X}^j_{B_t} \right\rangle_j$ of $\textbf{grad}_j$ via SMM protocol\\ 
		}
	}

	\For{$i = 1$ to $n$  \texttt{in parallel}}
	{
		\For{$j = 1$ to $n$}
		{
			$\mathcal{A}_i$ locally calculates the summation of all $j$-shares of $\textbf{grad}_i$, denoted as $\left\langle \textbf{grad}_i \right\rangle_j$\\ 
			$\mathcal{A}_i$ locally updates $\left\langle \textbf{w}_i \right\rangle_j$ by $\left\langle \textbf{w}_i \right\rangle_j \leftarrow \left\langle \textbf{w}_i \right\rangle_j - \frac{\alpha}{|B|} \cdot \left\langle \textbf{grad}_i \right\rangle_j$ \label{algo-proof-update}
		}
	}


}
\For{$i = 1$ to $n$ }
{
	\For{$j = 1$ to $n$ and $j \neq i$}
	{
		$\mathcal{A}_j$ sends $\left\langle \textbf{w}_i \right\rangle_j$ to $\mathcal{A}_i$ \\ 
	}
	$\mathcal{A}_i$ locally calculates the summation of $\{ \left\langle \textbf{w}_i \right\rangle_j \}_{j\in [n]}$, denoted as $\textbf{w}_i$
}
\end{algorithm}




\nosection{Secure Linear Regression Protocol} 
We first summarize the secure protocol for linear regression in Algorithm \ref{algo-verLR}. 
Since the data matrix is vertically partitioned, we only need to secretly share models (Line 4) and labels (Line 7) among participants. 
Each participant calculates shares of the prediction by using SMM protocol (Lines 8-12), and 
During model training, each participant gets a share of the prediction (Line 15), error (Line 16), and gradients (Line 26), updates the shared models (Line 27), and finally reconstruct their corresponding models (Line 35). 
Similar as the protocol in horizontally partitioned data in Algorithm 1, Algorithm \ref{algo-verLR} also works for any number of participants. 


\nosection{Secure Logistic Regression Protocol} 
Similar as the secure logistic regression protocol under horizontally partitioned data, one can use a $k$-order polynomial to approximate the logistic function. After it, the polynomial can be easily calculated using SMM protocol. 

\nosection{Security Proof} 
\begin{theorem}
	Algorithm \ref{algo-verLR} is secure against semi-honest adversaries, as in \textbf{Definition 1}.
\end{theorem}

\begin{proof}
	We skip the correctness proof of Algorithm \ref{algo-verLR} considering it is obvious. 
	To proof its security, for concise purpose, we use two parties as examples, i.e., $\mathcal{A}_1$ and $\mathcal{A}_2$, $\mathcal{A}_1$ has $\textbf{X}^1$ and $\mathcal{B}$ has $\textbf{X}^2$ and \textbf{y}. We construct two simulators $\mathcal{S}_{\mathcal{A}_1}$ and $\mathcal{S}_{\mathcal{A}_2}$, such that
	\begin{eqnarray}
		\{\mathcal{S}_{\mathcal{A}_1}(1^\kappa,\textbf{X}^1,\textbf{w}_1)\}_{\textbf{X}^1,\textbf{X}^2,\textbf{y},\kappa}\cong\{\mathsf{view}_{\mathcal{A}_1}(\textbf{X}^1,\textbf{X}^2,\textbf{y},\kappa)\}_{\textbf{X}^1,\textbf{X}^2,\textbf{y},\kappa},\\
		\{\mathcal{S}_{\mathcal{B}}(1^\kappa,\textbf{X}^2,\textbf{y},\textbf{w}_2)\}_{\textbf{X}^1,\textbf{X}^2,\textbf{y},\kappa}\cong\{\mathsf{view}_{\mathcal{B}}(\textbf{X}^1,\textbf{X}^2,\textbf{y},\kappa)\}_{\textbf{X}^1,\textbf{X}^2,\textbf{y},\kappa},
	\end{eqnarray}

	where $\mathsf{view}_{\mathcal{A}_1}$ and $\mathsf{view}_{\mathcal{A}_2}$ denotes the views of $\mathcal{A}_1$ and $\mathcal{A}_2$, respectively.
	We prove the above equations for a corrupted $\mathcal{A}_1$ and a corrupted $\mathcal{A}_2$, respectively.
	
	\paragraph{Corrupted $\mathcal{A}_1$.} In this case, we construct a probabilistic polynomial-time simulator $\mathcal{S}_{\mathcal{A}_1}$ that, when given the security parameter $\kappa$, $\mathcal{A}_1$'s input $\textbf{X}^1$ and output $\textbf{w}_1$, can simulate the view of $\mathcal{A}_1$ in the protocol execution. To this end, we first analyze $\mathcal{A}_1$'s view $\mathsf{view}_{\mathcal{A}_1}(\textbf{X}^1,\textbf{X}^2,\textbf{y},\kappa)$ in Algorithm \ref{algo-verLR}. 
	The messages obtained by $\mathcal{A}_1$ are consisted of three parts. 
	The first part is the messages sent before the training process, i.e., $\langle\textbf{y}_1\rangle,\langle\textbf{w}_2\rangle_1$; 
	the second part is the messages sent in the training process, which are from SMM protocols, i.e., 
	$\left\langle \textbf{X}^2_{B_t} \times \left\langle\textbf{w}_2\right\rangle_1 \right\rangle_2$, 
	$\left\langle \textbf{X}^1_{B_t} \times \left\langle\textbf{w}_1\right\rangle_2 \right\rangle_2$, 
	$\left\langle \left\langle \textbf{err} \right\rangle_1 ^t \times \textbf{X}^2_{B_t} \right\rangle_2$, and
	$\left\langle \left\langle \textbf{err} \right\rangle_2 ^t \times \textbf{X}^1_{B_t} \right\rangle_2$; 
	the third part is the message sent after the training process, i.e., $\langle\textbf{w}_1\rangle_2$. 
	Therefore, $\mathsf{view}_{\mathcal{A}_1}(\textbf{X}^1,\textbf{X}^2,\textbf{y},\kappa)$ consists of $\mathcal{A}_1$'s input $\textbf{X}^1$, the shares $\langle\textbf{y}\rangle_1$, $\langle\textbf{w}_2\rangle_1$,  
	$\left\langle \textbf{X}^2_{B_t} \times \left\langle\textbf{w}_2\right\rangle_1 \right\rangle_2$, 
	$\left\langle \textbf{X}^1_{B_t} \times \left\langle\textbf{w}_1\right\rangle_2 \right\rangle_2$, 
	$\left\langle \left\langle \textbf{err} \right\rangle_1 ^t \times \textbf{X}^2_{B_t} \right\rangle_2$, 
	$\left\langle \left\langle \textbf{err} \right\rangle_2 ^t \times \textbf{X}^1_{B_t} \right\rangle_2$, 
	and $\langle\textbf{w}_1\rangle_2$. 

	Given $\kappa$, $\textbf{X}^1$, and $\textbf{w}_1$, $\mathcal{S}_{\mathcal{A}_1}$ generates a simulation of $\mathsf{view}_{\mathcal{A}_1}(\textbf{X}^1,\textbf{X}^2,\textbf{y},\kappa)$ as the following steps.
	\begin{itemize}
		\item $\mathcal{S}_{\mathcal{A}_1}$ randomly selects shares $\langle\textbf{y}'\rangle_1$ and $\langle\textbf{w}'_2\rangle_1$. 

		\item $\mathcal{S}_{\mathcal{A}_1}$ simulates the SMM protocol, randomly generates $\left\langle \textbf{X}'^2_{B_t} \times \left\langle\textbf{w}'_2\right\rangle_1 \right\rangle_2$, 
		$\left\langle \textbf{X}'^1_{B_t} \times \left\langle\textbf{w}'_1\right\rangle_2 \right\rangle_2$, 
		$\left\langle \left\langle \textbf{err}' \right\rangle_1 ^t \times \textbf{X}'^2_{B_t} \right\rangle_2$, 
		$\left\langle \left\langle \textbf{err}' \right\rangle_2 ^t \times \textbf{X}'^1_{B_t} \right\rangle_2$, 
		and takes them as the output for $\mathcal{A}_1$ in SMM protocol.

		\item $\mathcal{S}_{\mathcal{A}_1}$ computes $\langle\textbf{w}'_1\rangle_2$ following Line \ref{algo-proof-update} in Algorithm \ref{algo-verLR} using the simulated results in previous steps. 

		\item $\mathcal{S}_{\mathcal{A}_1}$ generates a simulation of $\mathsf{view}_{\mathcal{A}_1}(\textbf{X}^1,\textbf{X}^2,\textbf{y},\kappa)$ by outputting $(\textbf{X}^1$, $\langle\textbf{y}'\rangle_1$, $\langle\textbf{w}'_2\rangle_1$, $\left\langle \textbf{X}'^2_{B_t} \times \left\langle\textbf{w}'_2\right\rangle_1 \right\rangle_2$, 
		$\left\langle \textbf{X}'^1_{B_t} \times \left\langle\textbf{w}'_1\right\rangle_2 \right\rangle_2$, 
		$\left\langle \left\langle \textbf{err}' \right\rangle_1 ^t \times \textbf{X}'^2_{B_t} \right\rangle_2$, 
		$\left\langle \left\langle \textbf{err}' \right\rangle_2 ^t \times \textbf{X}'^1_{B_t} \right\rangle_2$, $\langle\textbf{w}'_1\rangle_2)$.
	\end{itemize}
	
	Therefore, we have the following two equations:
	\begin{align*}
		&\mathsf{view}_{\mathcal{A}_1}(\textbf{X}^1,\textbf{X}^2,\textbf{y},\kappa)\\
		&=(\textbf{X}^1, \langle\textbf{y}\rangle_1, \langle\textbf{w}_2\rangle_1, \left\langle \textbf{X}^2_{B_t} \times \left\langle\textbf{w}_2\right\rangle_1 \right\rangle_2, 
		\left\langle \textbf{X}^1_{B_t} \times \left\langle\textbf{w}_1\right\rangle_2 \right\rangle_2, 
		\left\langle \left\langle \textbf{err} \right\rangle_1 ^t \times \textbf{X}^2_{B_t} \right\rangle_2, \\
		&\left\langle \left\langle \textbf{err} \right\rangle_2 ^t \times \textbf{X}^1_{B_t} \right\rangle_2, \langle\textbf{w}_1\rangle_2),\\
		&\mathcal{S}_{\mathcal{A}_1}(\textbf{X}^1,\textbf{X}^2,\textbf{y},\kappa)\\
		&=(\textbf{X}^1, \langle\textbf{y}'\rangle_1, \langle\textbf{w}'_2\rangle_1, \left\langle \textbf{X}'^2_{B_t} \times \left\langle\textbf{w}'_2\right\rangle_1 \right\rangle_2, 
		\left\langle \textbf{X}'^1_{B_t} \times \left\langle\textbf{w}'_1\right\rangle_2 \right\rangle_2, 
		\left\langle \left\langle \textbf{err}' \right\rangle_1 ^t \times \textbf{X}'^2_{B_t} \right\rangle_2, \\
		&\left\langle \left\langle \textbf{err}' \right\rangle_2 ^t \times \textbf{X}'^1_{B_t} \right\rangle_2, \langle\textbf{w}'_1\rangle_2).
	\end{align*}
	
	We note that the probability distributions of $\mathcal{A}_1$'s view and $\mathcal{S}_{\mathcal{A}_1}$'s output are computationally indistinguishable.
	This completes the proof in the case of corrupted $\mathcal{A}_1$.
	
	\paragraph{Corrupted $\mathcal{A}_2$.} In this case, we construct a probabilistic polynomial-time simulator $\mathcal{S}_{\mathcal{A}_2}$ that, when given the security parameter $\kappa$, $\mathcal{A}_2$'s input $\textbf{X}^2$, $\textbf{y}$ and output $\textbf{w}_2$, can simulate the view of $\mathcal{A}_2$ in the protocol execution. Following the proof of \textit{Corrupted $\mathcal{A}_1$}, one can simply proof that the probability distributions of $\mathcal{A}_2$'s view and $\mathcal{S}_{\mathcal{A}_2}$'s output are computationally indistinguishable. 
	This completes the proof in the case of corrupted $\mathcal{A}_2$.

\end{proof}

The above proof can be can be extended to multiple parties. Similarly, one can also prove that Algorithm \ref{algo-horLR} is secure against semi-honest adversaries. 

\nosection{Discussion} 
Before applying our proposed Algorithm \ref{algo-verLR} in practice, the batches in each SGD iteration must be aligned for the parties. This is why we need the step in Line \ref{algostep-batch-align}. The alignment operation can be done efficiently by using private set intersection \cite{pinkas2014faster}, which can match the samples in different datasets and keep the secure of these data at the same time.

\section{Experiments and Applications}


\begin{table}[t]
\label{exp-dataset}
\centering
\caption{Dataset statistics.}
\begin{tabular}{|c|c|c|c|c|}
  \hline
  Dataset & \textit{News}  & \textit{Blog} & \textit{Bank}  & \textit{APS}\\
  \hline
  \hline
  Number of feature & 61  & 180 & 17 & 171\\
  \hline
  Number of sample & ~~39,797~~  & ~~52,397~~ & 45,211 & 60,000 \\
  \hline
\end{tabular}
\end{table}

\subsection{Dataset and Data Split Description}
\noindent \textbf{Dataset description}. 
We use four public dataset to perform experiments, i.e., 
online news popularity dataset (\textit{News} for short) \cite{fernandes2015proactive}, BlogFeedback dataset (\textit{Blog} for short) \cite{buza2014feedback}, Bank Marketing dataset (\textit{Bank} for short) \cite{moro2014data}, and APS Failure dataset (\textit{APS} for short) \cite{Dua:2019}. 
The first two datasets are for linear regression task and the last two datasets are for logistic regression task. 
We summarize their statistics in Table 1. 
Note that for all the datasets, we normalize all the features and labels so that they are robust to regression tasks. 

\noindent \textbf{Data split}. 
For simplification, we only assume there are two parties. 
For Horizontally data split setting, we assume the two parties have the same number of samples. 
For vertically data split setting, we assume the two parties have the same number of features. 
Note that, our protocols are suitable for any number parties in practice. 

\subsection{Experimental Settings}

\nosection{Evaluation metric}
We use Root Mean Square Error (RSME) to evaluate the performance of (secure) linear regression models, and choose Area Under the ROC Curve (AUC) to evaluate the performance of (secure) logistic regression models. 


\nosection{Comparison methods}
We propose secure linear regression and logistic regression protocols for both Horizontally (H) and Vertically (V) partitioned data. 
Therefore, we compare with plaintext linear regression (LiRe) and logistic regression (LoRe) to study 
(1) whether they have the same accuracy, and (2) what is the difference of their running time. 
Besides, we apply two SMM protocols, i.e., with Trusted Initializer (TI) and withOut TI (OTI), for our proposed secure regressions. 
Thus, for ablation study, we use Sec-LiRe-TI-H/Sec-LiRe-OTI-H to denote secure linear regression model with/without trusted initializer under horizontally partitioned data, and use Sec-LiRe-TI-V/Sec-LiRe-OTI-V to denote secure linear regression model with/without trusted initializer under vertically partitioned data. 
Similarly, we have Sec-LoRe-TI-H, Sec-LoRe-OTI-H, Sec-LoRe-TI-V, and Sec-LoRe-OTI-V for logistic regression.

\nosection{Parameter setting}
For all the (secret sharing based) models, we set the batch size to $B=5$ and the number of iteration to 100, which means that we use mini-batch gradient descent to train the model. 
We also search the learning rate $\alpha$ in $\{0.001, 0.01, 0.1\}$ to find its best values.

\subsection{Comparison Results}


\begin{table}[t]\label{compare-news}
\centering
\caption{RMSE and running time of (secure) linear regressions on \textit{News} dataset.}
\begin{tabular}{|c|c|c|c|c|c|}
  \hline
  Model & LiRe  & Sec-LiRe-TI-V & Sec-LiRe-OTI-V & Sec-LiRe-TI-H & Sec-LiRe-OTI-H \\
  \hline
  \hline
  RMSE & 0.0096 & 0.0096 & 0.0096  & 0.0096 & 0.0096 \\
  \hline
  Time & 32.41 & 70.28 & 246.31 & 78.52  & 274.11 \\
  \hline
\end{tabular}
\end{table}

We use five-fold cross validation during comparison, and report the average results. 
We summarize the comparison results, including RMSE/AUC and running time (in seconds) in Tables 2-5. 
Note that we omit the offline Beaver triple generation time for the trusted-initializer based methods and use local area network. 
From them, we observe that:
\begin{itemize} [leftmargin=*] \setlength{\itemsep}{-\itemsep}
    \item Our proposed secure linear regression protocols have exactly the same performance with plaintext ones, and secure logistic regression protocols also have comparable performance with plaintext ones. This is because we use a 3-order polynomial to approximate the logistic function.
    \item The computation time of our proposed secure linear regression and logistic regression models are slower than plaintext ones, especially for the secure logistic regression models. This is because it needs more rounds of SMM protocols for secure logistic regression to calculate the 3-order polynomial. 
    Take Sec-LoRe-OTI-H on \textit{News} dataset for example, Sec-LoRe-OTI-H takes 9.06 times longer than plaintext LoRe, which is acceptable considering its ability of protecting data privacy. 
\end{itemize}

\subsection{Time Complexity Analysis}
We now study the time complexity of our proposed secure regression protocols, and report the results in Figure \ref{fig:effect}, where we use the same setting as in comparison. 
From it, we can see that with the increase of data size, the running time of our protocols scale linearly. 
This results demonstrate that our proposed model can be applied into large scale dataset.

\begin{table}[t]\label{compare-blog}
\centering
\caption{RMSE and running time of (secure) linear regressions on \textit{Blog} dataset.}
\begin{tabular}{|c|c|c|c|c|c|}
  \hline
  Model & LiRe  & Sec-LiRe-TI-V & Sec-LiRe-OTI-V & Sec-LiRe-TI-H & Sec-LiRe-OTI-H \\
  \hline
  \hline
  RMSE & 0.0125 & 0.0125 & 0.0125 & 0.0125  & 0.0125  \\
  \hline
  Time & 126.27 & 240.82 & 628.98 & 264.21  & 692.72 \\
  \hline
\end{tabular}
\end{table}

\begin{table}[t]\label{compare-bank}
\centering
\caption{AUC and running time of (secure) logistic regressions on \textit{Bank} dataset.}
\begin{tabular}{|c|c|c|c|c|c|}
  \hline
  Model & LoRe  & Sec-LoRe-TI-V & Sec-LoRe-OTI-V & Sec-LoRe-TI-H & Sec-LoRe-OTI-H \\
  \hline
  \hline
  AUC & 0.7849 & 0.7792 & 0.7792 & 0.7792 & 0.7792  \\
  \hline
  Time & 47.99 & 170.89 & 318.99 & 206.12 & 387.83 \\
  \hline
\end{tabular}
\end{table}

\begin{table}[t]\label{compare-aps}
\centering
\caption{AUC and running time of (secure) logistic regressions on \textit{APS} dataset.}
\begin{tabular}{|c|c|c|c|c|c|}
  \hline
  Model & LoRe  & Sec-LoRe-TI-V & Sec-LoRe-OTI-V & Sec-LoRe-TI-H & Sec-LoRe-OTI-H \\
  \hline
  \hline
  AUC & 0.9807 & 0.9749 & 0.9749 & 0.9749 & 0.9749  \\
  \hline
  Time & 151.72 & 607.09 & 1138.82 & 722.16 & 1374.45 \\
  \hline
\end{tabular}
\end{table}

\begin{figure}[t]
\centering 
\includegraphics[width=0.6\textwidth]{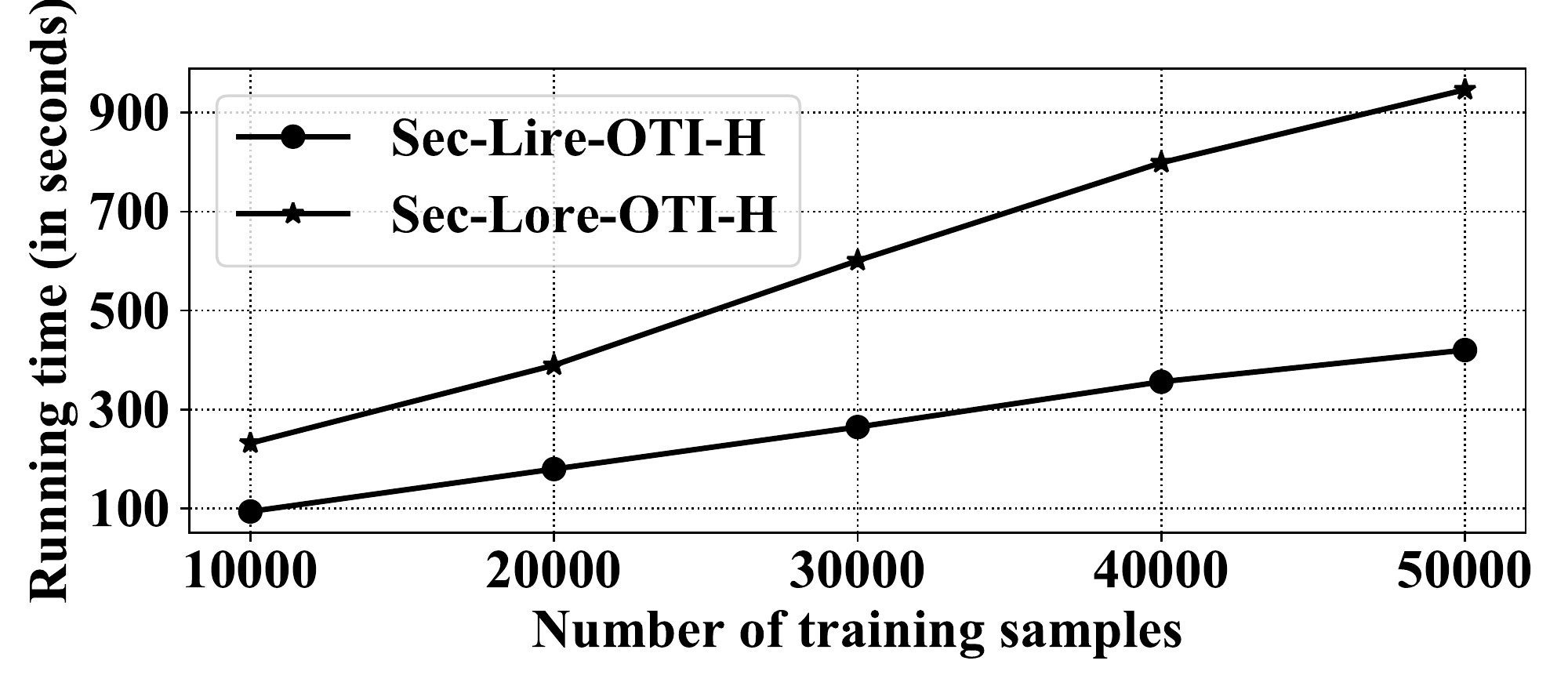} 
\caption{Running time (in seconds) with different data size. }
\label{fig:effect}
\end{figure}

\subsection{Applications}
Our proposed secure regression models have been successfully deployed in various tasks inside and outside Ant Financial, including intelligent marketing, risk control, and intelligent lending. 
For example, CDFinance\footnote{http://www.cdfinance.com.cn/en/index}, a bank in China, together with Ant Financial deployed secure logistic regression models, which not only significantly improved its risk control ability, but also transformed the traditional offline lending mode into an online automatic lending mode. 

\section{Conclusion and Future Work}
In this paper, we first made a computational reduction from the multiparty training of  regression models to secure multiparty matrix summation multiplication. 
Based on secret sharing schemes, we then proposed two secure regression algorithms for horizontally and vertically partitioned data respectively. 
We finally demonstrated the effectiveness and efficiency of our approach by experiments, and presented the real-world applications. In future, we would like to apply our proposed protocols into more machine learning algorithms and deploy them into more applications. 
\bibliographystyle{splncs04}
\bibliography{ref}

\begin{thebibliography}{10}
\providecommand{\url}[1]{\texttt{#1}}
\providecommand{\urlprefix}{URL }
\providecommand{\doi}[1]{https://doi.org/#1}

\bibitem{aono2016scalable}
Aono, Y., Hayashi, T., Trieu~Phong, L., Wang, L.: Scalable and secure logistic
  regression via homomorphic encryption. In: Proceedings of the Sixth ACM
  Conference on Data and Application Security and Privacy. pp. 142--144. ACM
  (2016)

\bibitem{buza2014feedback}
Buza, K.: Feedback prediction for blogs. In: Data analysis, machine learning
  and knowledge discovery, pp. 145--152. Springer (2014)

\bibitem{chen2020secure}
Chen, C., Li, L., Wu, B., Hong, C., Wang, L., Zhou, J.: Secure social
  recommendation based on secret sharing. arXiv preprint arXiv:2002.02088
  (2020)

\bibitem{chen2018privacy}
Chen, C., Liu, Z., Zhao, P., Zhou, J., Li, X.: Privacy preserving
  point-of-interest recommendation using decentralized matrix factorization.
  In: Thirty-Second AAAI Conference on Artificial Intelligence (2018)

\bibitem{chen2020practical}
Chen, C., Wu, B., Fang, W., Zhou, J., Wang, L., Qi, Y., Zheng, X.: Practical
  privacy preserving poi recommendation. arXiv preprint arXiv:2003.02834
  (2020)

\bibitem{chen2018logistic}
Chen, H., Gilad-Bachrach, R., Han, K., Huang, Z., Jalali, A., Laine, K.,
  Lauter, K.: Logistic regression over encrypted data from fully homomorphic
  encryption. BMC medical genomics  \textbf{11}(4), ~81 (2018)

\bibitem{de2017efficient}
De~Cock, M., Dowsley, R., Horst, C., Katti, R., Nascimento, A., Poon, W.S.,
  Truex, S.: Efficient and private scoring of decision trees, support vector
  machines and logistic regression models based on pre-computation. TDSC
  (2017)

\bibitem{demmler2015aby}
Demmler, D., Schneider, T., Zohner, M.: Aby-a framework for efficient
  mixed-protocol secure two-party computation. In: NDSS (2015)

\bibitem{Dua:2019}
Dua, D., Graff, C.: {UCI} machine learning repository (2017),
  \url{http://archive.ics.uci.edu/ml}

\bibitem{fernandes2015proactive}
Fernandes, K., Vinagre, P., Cortez, P.: A proactive intelligent decision
  support system for predicting the popularity of online news. In: Portuguese
  Conference on Artificial Intelligence. pp. 535--546. Springer (2015)

\bibitem{GasconSB0DZE17}
Gasc{\'{o}}n, A., Schoppmann, P., Balle, B., Raykova, M., Doerner, J., Zahur,
  S., Evans, D.: Privacy-preserving distributed linear regression on
  high-dimensional data. PoPETs  \textbf{2017}(4),  345--364 (2017)

\bibitem{goldreich2009foundations}
Goldreich, O.: Foundations of cryptography: volume 2, basic applications.
  Cambridge university press (2009)

\bibitem{hall2011secure}
Hall, R., Fienberg, S.E., Nardi, Y.: Secure multiple linear regression based on
  homomorphic encryption. Journal of Official Statistics  \textbf{27}(4), ~669
  (2011)

\bibitem{hazay2010efficient}
Hazay, C., Lindell, Y.: Efficient secure two-party protocols: Techniques and
  constructions. Springer Science \& Business Media (2010)

\bibitem{Karr2010a}
Karr, A.F.: Secure statistical analysis of distributed databases, emphasizing
  what we don{\textquoteright}t know. Journal of Privacy and Confidentiality
  \textbf{1}(2),  197--211 (2010)

\bibitem{kim2018secure}
Kim, M., Song, Y., Wang, S., Xia, Y., Jiang, X.: Secure logistic regression
  based on homomorphic encryption: Design and evaluation. JMIR medical
  informatics  \textbf{6}(2) (2018)

\bibitem{li2019federated}
Li, T., Sahu, A.K., Talwalkar, A., Smith, V.: Federated learning: Challenges,
  methods, and future directions. arXiv preprint arXiv:1908.07873  (2019)

\bibitem{lindell2005secure}
Lindell, Y.: Secure multiparty computation for privacy preserving data mining.
  In: Encyclopedia of Data Warehousing and Mining, pp. 1005--1009. IGI Global
  (2005)

\bibitem{liu2020privacy}
Liu, Y., Chen, C., Zheng, L., Wang, L., Zhou, J., Liu, G.: Privacy preserving
  pca for multiparty modeling. arXiv preprint arXiv:2002.02091  (2020)

\bibitem{mohassel2017secureml}
Mohassel, P., Zhang, Y.: Secureml: A system for scalable privacy-preserving
  machine learning. In: IEEE S\&P. pp. 19--38 (2017)

\bibitem{moro2014data}
Moro, S., Cortez, P., Rita, P.: A data-driven approach to predict the success
  of bank telemarketing. Decision Support Systems  \textbf{62},  22--31 (2014)

\bibitem{nikolaenko2013privacy}
Nikolaenko, V., Weinsberg, U., Ioannidis, S., Joye, M., Boneh, D., Taft, N.:
  Privacy-preserving ridge regression on hundreds of millions of records. In:
  IEEE S\&P. pp. 334--348 (2013)

\bibitem{pinkas2014faster}
Pinkas, B., Schneider, T., Zohner, M.: Faster private set intersection based on
  $\{$OT$\}$ extension. In: $\{$USENIX$\}$ Security. pp. 797--812 (2014)

\bibitem{shamir1979share}
Shamir, A.: How to share a secret. Communications of the ACM  \textbf{22}(11),
  612--613 (1979)

\bibitem{wu2019generalization}
Wu, B., Zhao, S., Chen, C., Xu, H., Wang, L., Zhang, X., Sun, G., Zhou, J.:
  Generalization in generative adversarial networks: A novel perspective from
  privacy protection. In: Advances in Neural Information Processing Systems.
  pp. 306--316 (2019)

\bibitem{yao1982protocols}
Yao, A.C.: Protocols for secure computations. In: FOCS. pp. 160--164. IEEE
  (1982)

\bibitem{yao1986generate}
Yao, A.C.C.: How to generate and exchange secrets. In: FOCS. pp. 162--167. IEEE
  (1986)

\bibitem{zheng2020industrial}
Zheng, L., Chen, C., Liu, Y., Wu, B., Wu, X., Wang, L., Wang, L., Zhou, J.,
  Yang, S.: Industrial scale privacy preserving deep neural network. arXiv
  preprint arXiv:2003.05198  (2020)

\bibitem{zhu2015ESP}
Zhu, Y., Takagi, T.: Efficient scalar product protocol and its
  privacy-preserving application. IJESDF  \textbf{7}(1),  1--19 (2015)

\end{thebibliography}

\end{document}